\documentclass{article}
\usepackage{arxiv}
\renewcommand{\shorttitle}{Approximate Universal Prediction under Interaction}
\renewcommand{\headeright}{}
\renewcommand{\undertitle}{}
\usepackage[T1]{fontenc}
\usepackage[utf8]{inputenc}
\usepackage{microtype}
\usepackage{amsmath,amssymb,amsthm,mathtools}
\usepackage{booktabs,tabularx,xcolor,enumitem,natbib}
\usepackage[colorlinks=true,allcolors=blue!45!black]{hyperref}
\newtheorem{theorem}{Theorem}\newtheorem{proposition}{Proposition}\newtheorem{lemma}{Lemma}\newtheorem{corollary}{Corollary}\newtheorem{example}{Example}\newtheorem{remark}{Remark}
\newcommand{\KL}{D_{\mathrm{KL}}}\newcommand{\TV}{\mathrm{TV}}\newcommand{\E}{\mathbb{E}}\newcommand{\cM}{\mathcal{M}}\newcommand{\Ber}{\operatorname{Bern}}\newcommand{\norm}[1]{\left\lVert#1\right\rVert_2}
\title{Approximate Universal Prediction under Interaction}
\author{Richard Cornelius Suwandi\\School of Artificial Intelligence\\
The Chinese University of Hong Kong, Shenzhen\\
\texttt{richardsuwandi@link.cuhk.edu.cn}}
\date{}
\begin{document}
\maketitle

\begin{abstract}
Bayesian mixtures, and Solomonoff induction in particular, provide an idealized model of prediction learned by minimizing population log loss, and they underlie universal agents such as AIXI. An agent that plans with an approximation of such a predictor decides where that predictor is evaluated, and a planner can be drawn toward branches on which its model is optimistic. We study which approximation guarantees survive this feedback. For normalized finite-horizon causal mixtures, the transfer coefficient factors into a policy likelihood ratio and task evidence available before each prediction, so task rarity and action-branch rarity can compound. Within one fixed universal mixture, infinitely many computable environments admit excess prediction error of order $\kappa\eta\,2^{K(\mu)}$, where $\eta$ is the training KL budget and $\kappa$ measures policy coverage. This order is sharp, and for each target a computable predictor attains it at a suitable finite horizon, so computability alone does not remove the penalty. For bounded returns, value error is controlled by the probability of entering a poorly covered branch plus a square-root term in $\eta$. Both terms are necessary, and finite maximum coverage ratios are not required. Motivated by the mentor override of Golden Handcuffs, we derive margins for preserving deferral decisions and conditional KL budgets with an unavoidable inverse-history-probability penalty, and a counterexample separates agreement of deferral decisions from task-wise event avoidance. The guarantees apply to any predictor architecture that meets explicit population approximation and coverage conditions, which gives concrete targets for implementing idealized universal agents.
\end{abstract}

\section{Introduction}
An agent that plans with a learned model of its environment decides which situations that model will face next. In the Bayesian ideal, exemplified by AIXI and its finite-horizon relatives, the model is a mixture over candidate environments, and its prediction error in the true environment is controlled by that environment's prior mass \citep{hutter2003}. We study a learned predictor that approximates such a mixture. Its residual error need not be spread evenly. It can sit on environments that the prior treats as rare, or on action branches that the data-collection process seldom visits, and a planner that optimizes against the learned model is drawn toward exactly the branches where its predictions are optimistic. This paper asks which quantities govern the transfer of a predictive approximation guarantee to the decisions of the agent that relies on it.

The question has a direct alignment motivation. Several proposals for safe agents are defined through ideal Bayesian values. Golden Handcuffs, for example, uses pessimistic predictions and a value-triggered mentor override to constrain an agent \citep{ebtekarcohen2026}. Any implementation replaces the ideal predictor with an approximation, which can change both the actions the agent favors and the moments at which it yields control. Understanding where approximation error goes under interaction is therefore a precondition for carrying such guarantees over to implementable agents.

\paragraph{From passive to interactive robustness.}
Memory-based meta-learning can amortize Bayesian prediction into a network whose parameters stay fixed while its predictions adapt to context \citep{ortega2019,graumoya2024}. \citet{liteng2026} analyze the robustness of the resulting approximate mixture predictors for normalized mixtures of sequence measures. Their Theorems~1 and~2 show that a trajectory-KL approximation bound $\eta$ gives, for each environment of prior mass $w_\mu$,
\[
S(\mu,\rho)=O\!\left(\log_2\frac1{w_\mu}+\frac{\eta}{w_\mu}\right),
\]
which for a normalized universal mixture becomes $O(K(\mu)+\eta 2^{K(\mu)})$. Their Theorem~3 shows that this order is attained by infinitely many deterministic environments within one fixed universal mixture, using a prefix that is rare under the mixture but certain under the target environment \citep[Lemma~2 and Appendix~C]{liteng2026}. Their Proposition~1 and Theorem~4 give an exact average-case decomposition with excess $(\ln2)\eta$, and their Theorem~5 extends it to dominated shifts of the prior.

Interaction adds a second source of rarity. An environment can be common while its deployment action branch is rarely explored during training. The two sources need not multiply at every prediction. What matters is the evidence for the task available before the prediction being evaluated. Revelation after a rare action can still amplify errors in later predictions. Section~\ref{sec:evidence} distinguishes these cases.

Even after deployment prediction is controlled, a small change in predicted returns can reverse a decision near a tie. A deployment branch absent from training also makes the corresponding maximum occupancy ratio infinite. We therefore analyze bounded-return planning through the probability of reaching poorly covered branches.

\paragraph{Contributions.}
We work with finite-horizon normalized causal Bayesian mixtures, a learned causal model $\rho$, and a population assumption $\KL(P_\xi^\beta\Vert P_\rho^\beta)\le\eta$ under a behavior policy $\beta$. The main results are the following.
\begin{enumerate}[leftmargin=*,itemsep=2pt]
\item \emph{Joint rarity, characterized at both extremes} (Section~\ref{sec:joint}). Target-task prediction error is controlled by a joint occupancy coefficient $\Gamma(\mu,\pi;\beta)$ (Theorem~\ref{thm:joint}). This coefficient factors exactly into a policy likelihood ratio and the evidence for $\mu$ in the history (Proposition~\ref{prop:evidence}), so task rarity is paid only through the observations that reveal the task. If all environments remain observationally indistinguishable before every evaluated prediction, $\Gamma\le\kappa_\pi$ independently of the task's prior mass (Corollary~\ref{cor:extremes}). In the other direction, within one fixed universal causal mixture and for one fixed pair of simple policies, infinitely many computable environments can incur excess prediction error of order $\kappa_\pi\eta\,2^{K(\mu)}$, matching the upper bound up to constants (Theorem~\ref{thm:universal}). This extends the passive lower bound of \citet{liteng2026} to interaction. A finite-table approximation preserves its order for computable predictors selected separately for each target and horizon (Corollary~\ref{cor:computablelower}).
\item \emph{Planning with coverage failures} (Section~\ref{sec:planning}). For bounded returns, the value error of any policy is at most the probability $\delta_\pi(c)$ of entering a branch whose occupancy ratio exceeds $c$, plus $\sqrt{(\ln2)c\eta/2}$ (Theorem~\ref{thm:rarecoverage}). Neither term can be removed (Proposition~\ref{prop:tailnecessary}), and the bound can give nontrivial plug-in planning guarantees when every maximum occupancy ratio is infinite. Counterexamples show that on-policy accuracy cannot control comparator values, that the square-root rate is unavoidable near ties, that a simply described action sequence can be exponentially rare under behavior, and that Bayesian value stability does not imply task-wise policy stability.
\item \emph{Approximate mentor deferral} (Section~\ref{sec:deferral}). We give a margin for preserving a finite-horizon deferral decision and a conservative threshold that never skips an ideal deferral. Proposition~\ref{prop:continuation} localizes the training KL to observed histories, including a prefix-free collection selected by a stopping rule. We show that the resulting inverse-probability cost of conditioning is unavoidable. An explicit counterexample separates agreement of override decisions from task-wise event avoidance.
\end{enumerate}
Appendix~\ref{app:additional} collects supporting results on discarded context, average-case prediction, simultaneous prior and policy shift, and episode-level data allocation. Appendix~\ref{app:computability} records a computability obstruction for one fixed computable predictor.

The upper bounds combine classical tools, namely the KL chain rule, change of measure, conditional-mean identities, and continuity of bounded returns. The novelty lies in identifying the right interactive quantities, characterizing when they compound, and proving that the compounding is unavoidable inside a universal mixture. The guarantees are stated for population probability laws under explicit approximation, coverage, and planning assumptions. They therefore apply to any learning procedure and architecture that meets those assumptions, and they give precise targets for such procedures. Section~\ref{sec:implications} records what the same statements imply for evaluation and deferral inside this model.

\subsection{Related work}
\paragraph{Universal prediction and computable approximations.}
\citet{hutter2003} establishes prediction and general decision-loss guarantees for Bayesian mixtures and Solomonoff prediction. Computable alternatives address a different part of the problem. \citet{filan2016} analyze speed priors and derive prediction-loss guarantees and computational bounds under assumptions on the generating measure. \citet{limdl2026} study additive and multiplicative approximation of a regularized minimum-description-length objective for sequential model selection. Their criterion constrains model selection at each observed prefix, which is a different notion of approximation from our population KL to a reference mixture under a behavior policy. Neural approaches instead amortize prediction from a task distribution \citep{ortega2019,graumoya2024,bloem2025}, and \citet{genewein2026} review the connection between log-loss training, compression, and universal prediction. \citet{young2024} explore the hypothesis that transformers approximate Solomonoff induction, and \citet{wanmei2025} propose an algorithmic-information reading of language-model training. These perspectives motivate our reference model but do not certify the population KL assumption for any deployed predictor. Our question is what follows once such a guarantee is available.

\paragraph{Approximate universal agents.}
Combining approximate universal prediction with planning has a substantial history. \citet{veness2011} introduce MC-AIXI, which combines a context-tree mixture with Monte Carlo planning, and their Theorems~1 and~2 relate model approximation to prediction and value estimates. Our setting instead starts from the KL error of an arbitrary learned model under a specified behavior policy and tracks its transfer to a selected task and deployment policy, which introduces task-mixture rarity and behavior-to-deployment occupancy as separate quantities. The gap between accurate on-policy prediction and successful control is also well known. \citet{leikehutter2015} exhibit failures induced by universal priors, and \citet{leike2016} give guarantees for an exploratory Thompson-sampling agent. \citet{wyethhutter2025embedded} show that extracting response predictions from a universal distribution over joint action-percept histories need not preserve the guarantees of a universal causal environment mixture. This matters when interpreting a model trained on whole transcripts, because modeling actions and responses jointly does not by itself establish our causal approximation assumption. \citet{meulemans2026} study embedded universal agents that predict their own actions. Our setup keeps the policy separate from the response model and does not address these self-prediction problems.

\paragraph{Coverage, model error, and planning.}
\citet{uehara2022} analyze pessimistic model-based offline reinforcement learning under realizability and partial coverage, and their guarantees can compare against a covered policy without requiring every candidate policy to be covered. Our plug-in planning bounds require control of both the comparator and the selected policy, either through finite occupancy ratios or through explicit coverage-failure probabilities. Value continuity under trajectory distance already appears in \citet[Lemma~2]{leike2016}, and \citet{lobel2024} sharpen simulation bounds for Markov models. Bayesian offline approaches provide complementary algorithmic context \citep{fellows2026,ni2026}. Model-based agents that plan with learned sequence models include the Trajectory Transformer \citep{janner2021} and Reasoning via Planning \citep{hao2023}. The present guarantees are deterministic consequences of a population KL assumption, rather than finite-sample rates or calibrated uncertainty estimates, and our gap corollary differs from the statistical fast-rate analysis of \citet{hukallus2021}.

\paragraph{Mentorship, uncertainty, and safety.}
\citet{ebtekarcohen2026} introduce Golden Handcuffs, a pessimistic universal Bayesian agent with mentor-guided exploration and a value-based override. They prove capability and event-avoidance guarantees under their assumptions and identify computable approximations as future work. Our analysis quantifies how model error affects a finite-horizon deferral decision and when conditional guarantees remain available at observed histories, and Section~\ref{sec:deferral} identifies what remains for an approximate version of their guarantees. A separate modeling issue concerns missing probability mass. \citet{wyethhutter2025ignorance} distinguish termination and ignorance interpretations of semimeasure loss and develop the corresponding value functions. We use normalized trajectory laws and a shared bounded return, so extending our results to semimeasures requires choosing between those decision models.

\section{Finite-horizon causal setup}\label{sec:setup}
Fix an integer horizon $H\ge1$ and finite nonempty action and percept alphabets $\mathcal A$ and $\mathcal E$. At time $t$, the agent chooses $a_t\in\mathcal A$ and then receives $e_t\in\mathcal E$. Write $h_{<t}=a_1e_1\cdots a_{t-1}e_{t-1}$ and $\tau=(a_{1:H},e_{1:H})$. A causal model $M$ assigns a normalized probability vector $M(\cdot\mid h,a)$ at every history-action pair. A policy $\pi$ assigns a probability vector $\pi(\cdot\mid h)$ and does not observe the hidden task label. Its trajectory law is
\begin{equation}
P_M^\pi(\tau)=\prod_{t=1}^H\pi(a_t\mid h_{<t})M(e_t\mid h_{<t},a_t).\label{eq:traj}
\end{equation}
Models and policies are defined at all histories, including those to which they assign probability zero.

Let $\cM$ be a countable class of such environments, with $w_\mu>0$ and $\sum_\mu w_\mu=1$. Define the environment likelihood of a history by
\[
L_\mu(h_{<t})=\prod_{s<t}\mu(e_s\mid h_{<s},a_s),
\qquad Z(h)=\sum_\nu w_\nu L_\nu(h).
\]
For $Z(h)>0$, the causal mixture is
\begin{equation}
 w_\mu(h)=\frac{w_\mu L_\mu(h)}{Z(h)},\qquad
 \xi(e\mid h,a)=\sum_\mu w_\mu(h)\mu(e\mid h,a).\label{eq:posterior}
\end{equation}
At $Z(h)=0$, choose any normalized conditional distribution. The prior is over whole environments, and reusing the prior weights at every step would generally give a different model. Multiplying the posterior conditionals in \eqref{eq:posterior} telescopes and yields, for every common policy,
\begin{equation}
 P_\xi^\pi=\sum_\mu w_\mu P_\mu^\pi,
 \qquad P_\xi^\pi\ge w_\mu P_\mu^\pi.\label{eq:causalmixture}
\end{equation}
The same identities hold for prefix marginals. Conditioning on the chosen action does not change the task posterior, because the policy factor is common to all tasks.

Let $\rho$ be a normalized learned causal model and $\beta$ a behavior policy. Our population assumption is
\begin{equation}
\KL(P_\xi^\beta\Vert P_\rho^\beta)\le\eta<\infty.\label{eq:trainkl}
\end{equation}
Here $\KL(P\Vert Q)=\sum_xP(x)\log_2(P(x)/Q(x))$ with the usual zero-probability conventions. All KL divergences and unqualified complexity quantities are in bits, and natural logarithms are written $\ln$. Assumption~\eqref{eq:trainkl} concerns trajectory laws rather than empirical training loss. A learned model and any policy selected from it are held fixed when the results are applied. If they are random, the results apply conditionally on a realization for which \eqref{eq:trainkl} holds, with a fresh task drawn from $w$.

Define the pre-percept occupancy
\[
d_{M,t}^\pi(h,a)=P_M^\pi(h_{<t}=h,a_t=a).
\]
For nonnegative ratios, set $0/0=0$ and $x/0=\infty$ for $x>0$. The policy and joint coverage coefficients are
\begin{equation}
\kappa_\pi=\max_{t,h,a}\frac{d_{\xi,t}^\pi(h,a)}{d_{\xi,t}^\beta(h,a)},
\qquad
\Gamma(\mu,\pi;\beta)=\max_{t,h,a}\frac{d_{\mu,t}^\pi(h,a)}{d_{\xi,t}^\beta(h,a)}.\label{eq:coverage}
\end{equation}
These are ratios on full histories, not on a compressed state representation. Marginal domination gives
\begin{equation}
\Gamma(\mu,\pi;\beta)\le\frac{\kappa_\pi}{w_\mu}.\label{eq:coarsegamma}
\end{equation}

\subsection{Interpreting the approximation assumption}\label{sec:agents}
Assumption~\eqref{eq:trainkl} is the interactive counterpart of the excess log loss that a predictive model attains on its training distribution. Suppose each training episode samples $\mu\sim w$ and follows the common behavior policy $\beta$, and define the population response-prediction loss
\begin{equation}
\mathcal L_\beta(\rho)
=\E_{\mu\sim w}\E_{\tau\sim P_\mu^\beta}
\left[-\sum_{t=1}^H\log_2\rho(e_t\mid h_{<t},a_t)\right].\label{eq:responseloss}
\end{equation}
The mixture identity and cancellation of the common policy factors give the exact relation
\begin{equation}
\mathcal L_\beta(\rho)-\mathcal L_\beta(\xi)
=\KL(P_\xi^\beta\Vert P_{\rho}^\beta).\label{eq:lossbridge}
\end{equation}
Indeed, the difference is the expectation under $P_\xi^\beta$ of $\sum_t\log_2[\xi(e_t\mid h_{<t},a_t)/\rho(e_t\mid h_{<t},a_t)]$. Thus \eqref{eq:trainkl} bounds excess population log loss for environment responses, and the ideal mixture minimizes this loss over normalized causal predictors. Zero excess loss forces agreement with $\xi$ on pairs with positive training occupancy, while predictions at unvisited pairs remain unconstrained by the objective. This last fact is the source of every coverage quantity below.

A sequence model trained on interaction transcripts is one instance. For a finite token alphabet, fix a maximum response length and a termination rule that forces completion by that length. The possible responses then form a finite percept alphabet. A normalized autoregressive model factors each response probability into token conditionals, so response log loss is the sum of token log losses. Equation~\eqref{eq:lossbridge} applies when these responses come from the specified interaction distribution and loss is summed per episode.

Normalized next-token probabilities alone do not guarantee termination. Unbounded responses require a separate extension of the finite-alphabet setup. Empirical loss, excess loss within a restricted model family, and perplexity on unrelated text do not certify \eqref{eq:trainkl}. The common-policy requirement also matters for logged data. A demonstrator using task information absent from the recorded history need not satisfy it.

The response model $\rho(e\mid h,a)$ and the action policy $\pi(a\mid h)$ play different mathematical roles even when one network implements both. Our planning results apply when the action procedure approximately maximizes expected return under $\rho$ within a specified policy class, with the planning error measured by $\zeta$ in Theorem~\ref{thm:planning}. A policy that generates actions directly is still a policy in our setup, but its action-prediction loss does not control the response-model KL. A further limitation of real predictors, a finite context or memory, enters as an irreducible lower bound on $\eta$ that is computed exactly in Proposition~\ref{prop:context} of Appendix~\ref{app:additional}.

\section{Joint rarity of tasks and policies}\label{sec:joint}
We first record the probability inequalities that fix the constants in later results.
\begin{lemma}[Squared distance and KL]\label{lem:l2}
For probability vectors $p,q$ on a finite alphabet,
\begin{equation}
\norm{p-q}^2\le(\ln2)\KL(p\Vert q),\qquad
\TV(p,q)\le\sqrt{\frac{\ln2}{2}\KL(p\Vert q)}.\label{eq:l2kl}
\end{equation}
\end{lemma}
\begin{proof}
Set $s=\sum_{e:p_e>q_e}(p_e-q_e)=\frac12\lVert p-q\rVert_1$. The positive and negative parts each have total mass $s$, so $\norm{p-q}^2\le2s^2$. Pinsker's inequality in natural units gives $2s^2\le(\ln2)\KL(p\Vert q)$, which proves both claims. Infinite KL makes the bounds immediate.
\end{proof}

\begin{proposition}[Policy-coverage transfer]\label{prop:transfer}
Under \eqref{eq:trainkl}, if $\kappa_\pi<\infty$, then
\begin{equation}
\KL(P_\xi^\pi\Vert P_\rho^\pi)\le\kappa_\pi\eta.\label{eq:transfer}
\end{equation}
\end{proposition}
\begin{proof}
Let $k_t(h,a)=\KL(\xi(\cdot\mid h,a)\Vert\rho(\cdot\mid h,a))$. The common policy factors cancel in the likelihood ratio, and the KL chain rule gives
\begin{align*}
\KL(P_\xi^\pi\Vert P_\rho^\pi)
&=\sum_{t,h,a}d_{\xi,t}^\pi(h,a)k_t(h,a)\\
&\le\kappa_\pi\sum_{t,h,a}d_{\xi,t}^\beta(h,a)k_t(h,a)
\le\kappa_\pi\eta.
\end{align*}
Terms at zero-occupancy pairs are omitted. Finite training KL ensures finite local KL wherever training occupancy is positive, and finite coverage transfers that support condition to deployment.
\end{proof}
The proposition holds for a policy selected using $\rho$, but it supplies no upper bound on that policy's coverage coefficient. In particular, optimizing $\rho$ can favor a branch that $\beta$ never visits.

For any pair of causal models, define the cumulative squared prediction error
\begin{equation}
S_H^\pi(M,N)=\sum_{t,h,a}d_{M,t}^\pi(h,a)
 \norm{M(\cdot\mid h,a)-N(\cdot\mid h,a)}^2.\label{eq:squared}
\end{equation}
The occupancy is always taken under the first model, so this quantity is not symmetric. It is bounded above by $2H$. The approximation component evaluated on a target task is
\begin{equation}
A_H^{\mu,\pi}(\xi,\rho)=\sum_{t,h,a}d_{\mu,t}^\pi(h,a)
\norm{\xi(\cdot\mid h,a)-\rho(\cdot\mid h,a)}^2.\label{eq:approxcomponent}
\end{equation}
This component is not generally the task-wise excess $S_H^\pi(\mu,\rho)-S_H^\pi(\mu,\xi)$, whose cross term need not vanish.

\begin{theorem}[Joint-rarity upper bound]\label{thm:joint}
Assume \eqref{eq:trainkl} and $\Gamma=\Gamma(\mu,\pi;\beta)<\infty$. Then
\begin{equation}
A_H^{\mu,\pi}(\xi,\rho)\le(\ln2)\Gamma\eta,\label{eq:approxbound}
\end{equation}
\begin{equation}
S_H^\pi(\mu,\rho)\le
\min\left\{2H,\ (\ln2)\left(\sqrt{\log_2\frac1{w_\mu}}+\sqrt{\Gamma\eta}\right)^2\right\}.\label{eq:jointbound}
\end{equation}
In particular, when $\kappa_\pi<\infty$,
\begin{equation}
S_H^\pi(\mu,\rho)\le2\ln2\left(\log_2\frac1{w_\mu}+\frac{\kappa_\pi\eta}{w_\mu}\right).\label{eq:taskcoarse}
\end{equation}
\end{theorem}
\begin{proof}
Lemma~\ref{lem:l2} and occupancy domination give
\[
A_H^{\mu,\pi}(\xi,\rho)
\le(\ln2)\sum_{t,h,a}d_{\mu,t}^\pi(h,a)k_t(h,a)
\le(\ln2)\Gamma\eta.
\]
Mixture domination and the common-policy chain rule imply
\begin{equation}
S_H^\pi(\mu,\xi)\le(\ln2)\KL(P_\mu^\pi\Vert P_\xi^\pi)
\le(\ln2)\log_2\frac1{w_\mu}.\label{eq:bayesbaseline}
\end{equation}
The triangle inequality in the Euclidean space indexed by $(t,h,a,e)$ with weights $d_{\mu,t}^\pi(h,a)$ gives
\[
\sqrt{S_H^\pi(\mu,\rho)}\le
\sqrt{S_H^\pi(\mu,\xi)}+\sqrt{A_H^{\mu,\pi}(\xi,\rho)}.
\]
This proves \eqref{eq:jointbound}. Equation~\eqref{eq:coarsegamma} and $(\sqrt{x}+\sqrt{y})^2\le2(x+y)$ prove \eqref{eq:taskcoarse}.
\end{proof}

\begin{remark}[Meaning of the joint coefficient]
On this finite history space, $\Gamma$ is the smallest constant satisfying
\[
\sum_{t,h,a}d_{\mu,t}^\pi(h,a)f_t(h,a)
\le\Gamma\sum_{t,h,a}d_{\xi,t}^\beta(h,a)f_t(h,a)
\]
for every nonnegative array $f$. Sufficiency follows term by term, and necessity follows by putting all loss on a pair with maximal ratio or on an uncovered pair. This characterizes transfer of arbitrary local losses. It does not show that $\Gamma$ is necessary for every restricted model family or every reward function.
\end{remark}

Appendix~\ref{app:additional} also gives a clipped version of \eqref{eq:approxbound} that does not require finite $\Gamma$ (Proposition~\ref{prop:clipped}). It distinguishes a small amount of uncovered occupancy from pervasive coverage failure, in the same spirit as the planning result of Section~\ref{sec:rarecoverage}.

\subsection{The evidence factorization}\label{sec:evidence}
The coarse bound \eqref{eq:taskcoarse} multiplies the inverse prior mass of the task by the policy coefficient. The next proposition shows exactly when that product is paid. For a history $h$ with $Z(h)>0$, define the \emph{evidence} for $\mu$ in $h$ as
\[
\lambda_\mu(h)=\log_2\frac{L_\mu(h)}{Z(h)}=\log_2\frac{w_\mu(h)}{w_\mu}\le\log_2\frac1{w_\mu}.
\]
This is the logarithmic change from prior to posterior, measured in bits, and it can be negative. Set $\lambda_\mu(h)=-\infty$ when $L_\mu(h)=0<Z(h)$, with $2^{-\infty}=0$. The upper limit is attained only when the posterior of $\mu$ equals one.

\begin{proposition}[Evidence form of the joint coefficient]\label{prop:evidence}
At every pair with $d_{\xi,t}^\beta(h,a)>0$,
\begin{equation}
\frac{d_{\mu,t}^\pi(h,a)}{d_{\xi,t}^\beta(h,a)}
=\frac{d_{\xi,t}^\pi(h,a)}{d_{\xi,t}^\beta(h,a)}\;2^{\lambda_\mu(h)}.\label{eq:jointfactorization}
\end{equation}
Consequently, if $\kappa_\pi<\infty$ and $\mathcal S_\mu^\pi$ denotes the set of pairs with $d_{\mu,t}^\pi(h,a)>0$, then
\begin{equation}
\log_2\Gamma(\mu,\pi;\beta)
=\max_{(t,h,a)\in\mathcal S_\mu^\pi}\left[\log_2\frac{d_{\xi,t}^\pi(h,a)}{d_{\xi,t}^\beta(h,a)}+\lambda_\mu(h)\right]
\le\log_2\kappa_\pi+\max_{(t,h,a)\in\mathcal S_\mu^\pi}\lambda_\mu(h).\label{eq:evidencebound}
\end{equation}
\end{proposition}
\begin{proof}
Write $Q^\pi(h)$ for the product of the policy factors along $h$. Then $d_{\mu,t}^\pi(h,a)=L_\mu(h)Q^\pi(h)\pi(a\mid h)$ and, by \eqref{eq:causalmixture}, $d_{\xi,t}^\pi(h,a)=Z(h)Q^\pi(h)\pi(a\mid h)$. Positive training occupancy implies $Z(h)>0$, so $d_{\mu,t}^\pi(h,a)=2^{\lambda_\mu(h)}d_{\xi,t}^\pi(h,a)$, including pairs that $\pi$ does not visit. Dividing by $d_{\xi,t}^\beta(h,a)$ proves \eqref{eq:jointfactorization}. If $\kappa_\pi<\infty$, every pair in $\mathcal S_\mu^\pi$ has $d_{\xi,t}^\pi(h,a)\ge w_\mu d_{\mu,t}^\pi(h,a)>0$ and hence positive training occupancy, while pairs outside $\mathcal S_\mu^\pi$ contribute zero to $\Gamma$. Taking the maximum proves \eqref{eq:evidencebound}.
\end{proof}

In words, the joint coefficient charges the policy ratio at a pair together with the evidence for the task \emph{accumulated before that pair}. The prior mass of the task enters only through this evidence. The coarse bound \eqref{eq:coarsegamma} replaces the evidence by its maximum $\log_2(1/w_\mu)$ and the policy ratio by its maximum $\log_2\kappa_\pi$, and it is tight only when both maxima occur at the same pair.

\begin{corollary}[The two extremes]\label{cor:extremes}
Assume $\kappa_\pi<\infty$.
\begin{enumerate}[label=(\roman*),leftmargin=*,itemsep=1pt]
\item \emph{Late revelation.} If every environment in $\cM$ assigns the same likelihood as $\mu$ to every history at which $\pi$ acts with positive probability under $\mu$, then $\Gamma(\mu,\pi;\beta)\le\kappa_\pi$, independently of $w_\mu$. This holds in particular when all environments generate the same percept law before the final round.
\item \emph{Early revelation.} If some pair attaining $\kappa_\pi$ has posterior $w_\mu(h)=1$, then $\Gamma(\mu,\pi;\beta)=\kappa_\pi/w_\mu$.
\end{enumerate}
\end{corollary}
\begin{proof}
In case (i), $L_\nu(h)=L_\mu(h)$ for every $\nu$ gives $Z(h)=L_\mu(h)$ and $\lambda_\mu(h)=0$ on the relevant histories, and \eqref{eq:evidencebound} gives the claim. If all environments share the percept law before round $H$, then all likelihoods of histories $h_{<t}$ with $t\le H$ coincide. In case (ii), the evidence at the attaining pair is $\log_2(1/w_\mu)$, so \eqref{eq:jointfactorization} gives a ratio of $\kappa_\pi/w_\mu$ there, and \eqref{eq:coarsegamma} gives the reverse inequality.
\end{proof}

Case (i) includes horizon one, where occupancies precede any observation and $\Gamma=\kappa_\pi$. Its conclusion requires indistinguishability before every evaluated prediction. Revelation after a rare action is not enough, as the next example shows.

\begin{example}[Revelation after a rare action can still amplify error]\label{ex:laterevelation}
Take $H=3$, two environments $\mu_1,\mu_0$ with weights $w,1-w$, and $w\in(0,1)$. At round one, both emit $0$ after either action. Deployment chooses $r$, while behavior chooses $r$ with probability $\epsilon\in(0,1]$ and $s$ otherwise. Both policies choose $s$ thereafter. At round two, $\mu_1$ emits $1$ and $\mu_0$ emits $0$. At round three, both emit a fair bit.

The rare action precedes task revelation. Yet the round-three pair with history $(r,0,s,1)$ has target occupancy one and training occupancy $w\epsilon$. Consequently, $\kappa_\pi=1/\epsilon$ and $\Gamma(\mu_1,\pi;\beta)=1/(w\epsilon)$. Let $\rho$ equal $\xi$ except that it predicts $\Ber(1/2+\delta)$ at this pair, with $0<\delta\le1/4$. Its training KL and target excess are
\[
\eta=w\epsilon\left[-\frac12\log_2(1-4\delta^2)\right],
\qquad
S_3^\pi(\mu_1,\rho)-S_3^\pi(\mu_1,\xi)=2\delta^2.
\]
Thus the approximation penalty has order $\Gamma\eta$ despite revelation occurring after the rare action.
\end{example}

A two-stage instance attaining \eqref{eq:approxbound} up to the constant $3/4$ appears in Proposition~\ref{prop:jointlower}. The next subsection establishes the same dependence, up to machine constants, inside one fixed universal mixture.

\subsection{Sharpness inside one universal causal mixture}\label{sec:universal}
We now specialize to a universal prior on infinite causal environments and then restrict their laws to each finite horizon. Let $\mathcal E=\{0,1\}$, and let the finite alphabet $\mathcal A$ contain distinct actions $s$ and $r$. Fix a universal machine $U$ with a prefix-free program code, an infinite tape of fair random bits, and a causal action interface. The machine receives one action at a time and may then emit a percept.

A chronological law specifies $\nu(e_{1:t}\Vert a_{1:t})$ for each finite action-percept sequence, consistently across horizons and without dependence on future actions. Call such a law \emph{estimable} if a program samples it and almost surely produces a percept after each action, for every supplied action sequence. Identify programs that induce the same chronological law. At histories with positive environment likelihood, use the conditional ratios of that law. At zero-likelihood histories, use the uniform percept distribution. This fixed completion determines a kernel at every history but does not change trajectory probabilities or occupancy-weighted losses.

Let $\cM_U$ contain the distinct estimable chronological laws, let $K(\nu)$ be the shortest program length for the law, and set
\[
w_\nu=\frac{2^{-K(\nu)}}{Z_U},\qquad Z_U=\sum_{\nu\in\cM_U}2^{-K(\nu)}\le1,
\]
with $\xi_U$ the resulting causal mixture. One shortest program per distinct law forms a subset of the prefix-free program code, which gives the Kraft bound. The fair-coin environment shows $Z_U>0$ and gives $\xi_U$ positive probability for every finite percept string under every action string. This is the interactive analogue of the normalized universal mixture of \citet[Section~2]{liteng2026}. Since $w_\mu\ge2^{-K(\mu)}$, Theorem~\ref{thm:joint} gives
\begin{equation}
S_H^\pi(\mu,\rho)\le2\ln2\left(K(\mu)+\kappa_\pi\eta\,2^{K(\mu)}\right)\label{eq:universalupper}
\end{equation}
for every $\mu\in\cM_U$ and every policy with $\kappa_\pi<\infty$.

The lower bound uses two fixed policies. Following the proof of \citet[Theorem~3]{liteng2026}, let
\[
\mathcal C=\bigcup_{m\ge1}\mathcal C_m,\qquad \mathcal C_m=\{1^m0x:x\in\{0,1\}^m\},
\]
which is a decidable prefix-free set of percept strings. Both policies take action $s$ until the percept string observed so far belongs to $\mathcal C$. From the following round on, $\pi_{\mathcal C}$ always takes $r$. The behavior policy $\beta_\epsilon$ instead draws a single coin at that point and repeats its outcome, taking $r$ in every remaining round with probability $\epsilon$ and $s$ in every remaining round otherwise. Neither policy depends on the target environment or on $\eta$.

\begin{theorem}[Joint rarity is unavoidable in a universal mixture]\label{thm:universal}
Fix $\epsilon\in(0,1]$ and $\eta>0$, and let $\pi=\pi_{\mathcal C}$ and $\beta=\beta_\epsilon$. There is a constant $C_U\ge1$, depending only on the reference machines, and an infinite sequence of distinct environments $\mu\in\cM_U$ with the following property. For each such $\mu$ there are a horizon $H$ and a normalized causal model $\rho$ with $\KL(P_{\xi_U}^\beta\Vert P_\rho^\beta)\le\eta$ at horizon $H$, $\kappa_\pi=1/\epsilon$, and
\begin{equation}
S_H^\pi(\mu,\rho)-S_H^\pi(\mu,\xi_U)
\ge\frac{\kappa_\pi\eta}{4C_U\,w_\mu}
\ge\frac{\Gamma(\mu,\pi;\beta)\,\eta}{4C_U}.\label{eq:universallower}
\end{equation}
Equivalently, the excess prediction error is at least $(Z_U/4C_U)\,\kappa_\pi\eta\,2^{K(\mu)}$, which matches \eqref{eq:universalupper} up to constants.
\end{theorem}
\begin{proof}
\emph{Step 1. Mixture probability of strings in $\mathcal C$.} For a percept string $y$ of length $n$, let
\[
q(y)=\sum_{\nu\in\cM_U}w_\nu\,\nu(y\Vert s^n)
\]
be the probability that $\xi_U$ produces $y$ when action $s$ is taken in every round, where $\nu(y\Vert s^n)=\prod_{t\le n}\nu(y_t\mid y_{<t},s)$ with actions suppressed from the notation. Let $T(y)$ be the sum over all self-delimiting programs $p$ of $2^{-|p|}$ times the probability, over the random bits, that $U$ run on $p$ with the constant action stream $s$ emits a percept string beginning with $y$. Assigning each $\nu$ one of its shortest programs gives $Z_Uq(y)\le T(y)$. A standard dovetailing argument shows that $T$ is lower semicomputable. For each program, the events that the output begins with distinct strings of the prefix-free set $\mathcal C$ are disjoint, so $\sum_{y\in\mathcal C}T(y)\le\sum_p2^{-|p|}\le1$ by Kraft's inequality. Hence $T(y)\mathbf 1[y\in\mathcal C]$ is a lower semicomputable discrete semimeasure, and the coding theorem \citep[Corollary~4.3.1]{livitanyi2008}, in the form used by \citet[Lemma~4]{liteng2026}, gives a constant $c_U$ with
\begin{equation}
Z_U\,q(y)\le2^{c_U-K(y)}\qquad\text{for every }y\in\mathcal C,\label{eq:codingbound}
\end{equation}
where $K(y)$ is the prefix complexity of $y$ with respect to a fixed reference prefix machine.

\emph{Step 2. Target environments.} Since $\mathcal C_m$ has $2^m$ elements and fewer than $2^m$ programs are shorter than $m$, there is $y_m\in\mathcal C_m$ with $K(y_m)\ge m$. Let $\mu_m$ emit $y_m$ and then $0$ forever, ignoring the actions and the random bits. Then $\mu_m$ is estimable and $K(\mu_m)\le K(y_m)+c'$ for a constant $c'$. Because $\mathcal C$ is prefix-free, no proper prefix of $y_m$ lies in $\mathcal C$, so both policies take $s$ in each of the first $|y_m|$ rounds on any history whose percepts begin with $y_m$. Therefore
\[
q_m:=q(y_m)=P_{\xi_U}^\beta(e_{1:|y_m|}=y_m)\ge w_{\mu_m}
\]
by mixture domination. Combining this with \eqref{eq:codingbound} gives $2^{-K(\mu_m)}\le Z_Uq_m\le2^{c_U-K(y_m)}$. Hence $K(\mu_m)\ge m-c_U$, the environments $\mu_m$ are distinct for infinitely many $m$, $w_{\mu_m}\to0$, and
\begin{equation}
w_{\mu_m}\le q_m\le C_U\,w_{\mu_m},\qquad C_U=2^{c_U+c'}.\label{eq:qcompare}
\end{equation}

\emph{Step 3. Policy coverage.} Before $\mathcal C$ is entered, the two policies act identically, so the occupancy ratio is one. Afterwards, $\pi$ assigns no mass to $s$, and conditioned on its coin taking the value $r$, $\beta$ produces the same actions as $\pi$. Since percept laws given actions do not depend on the policy, every later pair $(h,r)$ satisfies $d_{\xi,t}^\beta(h,r)=\epsilon\, d_{\xi,t}^\pi(h,r)$. Hence $\kappa_\pi=1/\epsilon$ at every horizon $H>|y_m|$.

\emph{Step 4. The approximate model.} Fix $m$ and let $N=\lfloor\eta/(\epsilon q_m)\rfloor$ and $H=|y_m|+N$. For $0\le\ell<N$, let $h_\ell$ be the history consisting of the percepts $y_m$ under action $s$, followed by $\ell$ rounds of action $r$ and percept $0$. Let $\rho$ agree with $\xi_U$ except that $\rho(\cdot\mid h_\ell,r)=\Ber(1/2)$ for $0\le\ell<N$. Every binary distribution has KL at most one bit from $\Ber(1/2)$. Reaching $(h_\ell,r)$ under $\beta$ requires the percepts $y_m$ and the coin value $r$, which is independent of the percepts, so $d_{\xi_U}^\beta(h_\ell,r)\le\epsilon q_m$. The chain rule then gives
\[
\KL(P_{\xi_U}^\beta\Vert P_\rho^\beta)\le\sum_{\ell<N}d_{\xi_U}^\beta(h_\ell,r)\le N\epsilon q_m\le\eta.
\]

\emph{Step 5. Excess error.} Under $\mu_m$ and $\pi$, each pair $(h_\ell,r)$ is reached with probability one, the true conditional is the point mass $\delta_0$, and $\norm{\delta_0-\Ber(1/2)}^2=1/2$. All other conditionals of $\rho$ equal those of $\xi_U$. Using \eqref{eq:bayesbaseline},
\begin{align*}
S_H^\pi(\mu_m,\rho)-S_H^\pi(\mu_m,\xi_U)
&=\sum_{\ell<N}\left[\tfrac12-\norm{\delta_0-\xi_U(\cdot\mid h_\ell,r)}^2\right]\\
&\ge\frac N2-S_H^\pi(\mu_m,\xi_U)
\ge\frac{\eta}{2\epsilon C_Uw_{\mu_m}}-\frac12-(\ln2)\log_2\frac1{w_{\mu_m}},
\end{align*}
where the last step uses $N\ge\eta/(\epsilon q_m)-1$ and \eqref{eq:qcompare}. Because $w_{\mu_m}\to0$, the subtracted terms are at most half of the leading term for all sufficiently large $m$, which gives the first inequality in \eqref{eq:universallower} with $\kappa_\pi=1/\epsilon$. The second follows from \eqref{eq:coarsegamma}. Finally, $1/w_\mu=Z_U2^{K(\mu)}$.
\end{proof}

The theorem says that the coarse product $\kappa_\pi/w_\mu$ is not an artifact of a worst-case finite class. It is paid, up to a constant depending only on the reference machines, by infinitely many computable environments of one fixed universal mixture, under one fixed pair of policies that are computable whenever $\epsilon$ is. The construction realizes the large-evidence regime up to a constant, without requiring posterior probability one. The prefix $y_m$ carries about $K(\mu_m)$ bits of evidence for $\mu_m$ before the rare branch begins, so by \eqref{eq:qcompare} the posterior of $\mu_m$ at $h_0$ is at least $1/C_U$, and the policy penalty and the task penalty fall on the same pairs. With $\epsilon=1$ the policies coincide, and the theorem becomes a finite-horizon interactive form of \citet[Theorem~3]{liteng2026}. Since the excess grows with $1/w_\mu$, the total error also satisfies $S_H^\pi(\mu,\rho)=\Omega(K(\mu)+\kappa_\pi\eta\,2^{K(\mu)})$ for all sufficiently large $m$.

The horizon grows with $1/w_\mu$, as it must, because squared error never exceeds $2H$. The displayed construction uses $\xi_U$ and need not be computable. At each finite horizon, however, rational approximation preserves its lower-bound order. The following corollary makes this distinction explicit.

\begin{corollary}[Computable witnesses at finite horizons]\label{cor:computablelower}
For the fixed mixture and policies of Theorem~\ref{thm:universal}, and every $\eta>0$, infinitely many distinct $\mu\in\cM_U$ admit a horizon $H$ and a uniformly computable causal predictor $\bar\rho$ such that
\begin{equation}
\KL(P_{\xi_U}^{\beta}\Vert P_{\bar\rho}^{\beta})\le\eta,\qquad
S_H^\pi(\mu,\bar\rho)-S_H^\pi(\mu,\xi_U)
\ge\frac{\kappa_\pi\eta}{16C_Uw_\mu}.\label{eq:computablelower}
\end{equation}
The predictor may depend on $\mu$, $H$, $\eta$, and $\epsilon$. No bound on its description length or runtime is asserted.
\end{corollary}
\begin{proof}
Apply Theorem~\ref{thm:universal} with budget $\eta/2$. For each target in its infinite sequence, let $\rho_0$ be the constructed model at horizon $H$ and write
\[
D=\frac{\kappa_\pi\eta}{8C_Uw_\mu}.
\]
Its training KL is at most $\eta/2$ and its excess squared error is at least $D$. All conditionals of $\rho_0$ through $H$ are strictly positive: unchanged conditionals come from the full-support mixture, and the replacements are fair coins.

Fix $\alpha\in(0,1)$. The history tree through $H$ is finite. Density of rational probability vectors in the interior of each simplex therefore gives rational conditionals satisfying
\[
(1-\alpha)\rho_0(e\mid h,a)\le\bar\rho(e\mid h,a)
\le(1+\alpha)\rho_0(e\mid h,a)
\]
at every pair through $H$. Store these conditionals in a finite table and use a fair coin after $H$. This defines a predictor whose conditional probabilities are uniformly computable from the history and action.

The lower relative bound gives
\[
\KL(P_{\xi_U}^{\beta}\Vert P_{\bar\rho}^{\beta})
\le\KL(P_{\xi_U}^{\beta}\Vert P_{\rho_0}^{\beta})
+H\log_2\frac1{1-\alpha}.
\]
For probability vectors $p,q,r$,
\[
\bigl|\norm{p-r}^2-\norm{p-q}^2\bigr|
=|\langle q-r,2p-q-r\rangle|
\le2\lVert q-r\rVert_1.
\]
Our relative approximation implies $\lVert\bar\rho-\rho_0\rVert_1\le\alpha$ at every pair. Hence the change in $S_H^\pi(\mu,\rho_0)$ is at most $2H\alpha$. Choose $\alpha$ small enough that $H\log_2(1/(1-\alpha))\le\eta/2$ and $2H\alpha\le D/2$. The training KL is then at most $\eta$, while the excess is at least $D/2$, proving \eqref{eq:computablelower}.
\end{proof}

The corollary follows from Theorem~\ref{thm:universal} by a finite-dimensional approximation, and its message is that computability alone does not remove the multiplicative penalty. The witnesses are chosen separately for each target and horizon, and Appendix~\ref{app:computability} shows that no single computable predictor keeps bounded KL across all horizons. Any computational restriction that changes the conclusion must therefore involve more than computability, such as a uniform resource bound or a constrained learning procedure.

\section{From prediction to planning}\label{sec:planning}
Let $R(\tau)\in[0,1]$ be a fixed return functional shared by all models, and set $V_M^\pi=\E_{P_M^\pi}R$. Rewards may be encoded in percepts. For example, $R=H^{-1}\sum_{t=1}^H r_t$ with $r_t\in[0,1]$ satisfies the normalization, and for an unnormalized sum of $H$ such rewards the value bounds below are multiplied by $H$.

The passage from trajectory prediction to bounded value is classical and appears in Theorem~2 of \citet{veness2011} and Lemma~2 of \citet{leike2016}. We state the reduction with explicit behavior-coverage and planning-error terms so that it can be combined with the preceding analysis.

\begin{theorem}[Bayesian planning stability]\label{thm:planning}
Let $\Pi$ be a policy class with $\pi^\star\in\arg\max_{\pi\in\Pi}V_\xi^\pi$. Suppose $\widehat\pi\in\Pi$ satisfies
\[
V_\rho^{\widehat\pi}\ge\sup_{\pi\in\Pi}V_\rho^\pi-\zeta,
\qquad \zeta\ge0.
\]
Under \eqref{eq:trainkl}, if $\kappa_{\pi^\star}$ and $\kappa_{\widehat\pi}$ are finite, then
\begin{equation}
0\le V_\xi^{\pi^\star}-V_\xi^{\widehat\pi}
\le\min\left\{1,\ \zeta+\sqrt{\frac{\ln2}{2}\eta}
\left(\sqrt{\kappa_{\pi^\star}}+\sqrt{\kappa_{\widehat\pi}}\right)\right\}.\label{eq:planning}
\end{equation}
For exact planning and both coefficients at most $\kappa$, the bound is $\min\{1,\sqrt{2\ln2\,\kappa\eta}\}$.
\end{theorem}
\begin{proof}
Since $0\le R\le1$, Lemma~\ref{lem:l2} and Proposition~\ref{prop:transfer} yield
\[
|V_\xi^\pi-V_\rho^\pi|\le\TV(P_\xi^\pi,P_\rho^\pi)
\le\sqrt{\frac{\ln2}{2}\kappa_\pi\eta}.
\]
Inserting and subtracting the two learned-model values gives
\begin{align*}
V_\xi^{\pi^\star}-V_\xi^{\widehat\pi}
={}&(V_\xi^{\pi^\star}-V_\rho^{\pi^\star})
 +(V_\rho^{\pi^\star}-V_\rho^{\widehat\pi})
 +(V_\rho^{\widehat\pi}-V_\xi^{\widehat\pi}).
\end{align*}
The middle term is at most $\zeta$, and the other terms are bounded by the preceding inequality. The trivial cap is one.
\end{proof}
For finite alphabets and finite horizon, a deterministic optimal policy exists over the unrestricted policy class by backward induction on the finite history tree. For an arbitrary restricted class, attainment is an assumption of the theorem.

\subsection{Planning with rare coverage failures}\label{sec:rarecoverage}
A maximum occupancy ratio is infinite as soon as a deployment policy visits any pair absent from training. For bounded episode returns, one can instead charge the probability of reaching poorly covered pairs. For $c\ge1$, define
\begin{align}
\mathcal B_{\pi,c}
&=\{(t,h,a):d_{\xi,t}^\pi(h,a)>c\,d_{\xi,t}^\beta(h,a)\},\notag\\
\delta_\pi(c)
&=P_\xi^\pi\!\left(\text{some }(t,\mathsf H_{<t},\mathsf A_t)
\text{ belongs to }\mathcal B_{\pi,c}\right),\label{eq:coveragetail}
\end{align}
where $\mathsf H_{<t}$ and $\mathsf A_t$ denote the random history and action. These quantities depend on the reference environment law and are not claimed to be estimable from an arbitrary finite dataset.

\begin{theorem}[Value stability with a coverage-failure probability]\label{thm:rarecoverage}
Under \eqref{eq:trainkl}, every policy $\pi$ and every $c\ge1$ satisfy
\begin{equation}
|V_\xi^\pi-V_\rho^\pi|
\le\TV(P_\xi^\pi,P_\rho^\pi)
\le\min\left\{1,\ \delta_\pi(c)+\sqrt{\frac{\ln2}{2}c\eta}\right\}.\label{eq:tailvalue}
\end{equation}
Define
\[
b_\pi=\min\left\{1,\ \inf_{c\ge1}
\left[\delta_\pi(c)+\sqrt{\frac{\ln2}{2}c\eta}\right]\right\}.
\]
If $\pi^\star$ and $\widehat\pi$ satisfy the optimality and $\zeta$-approximate planning conditions of Theorem~\ref{thm:planning}, then
\begin{equation}
0\le V_\xi^{\pi^\star}-V_\xi^{\widehat\pi}
\le\min\{1,\ \zeta+b_{\pi^\star}+b_{\widehat\pi}\}.\label{eq:tailplanning}
\end{equation}
No finite maximum occupancy ratio is required.
\end{theorem}
\begin{proof}
Let $\sigma$ be the first round at which the history-action pair belongs to $\mathcal B_{\pi,c}$, with $\sigma=H+1$ if this never occurs. Construct a causal model $\widetilde\xi$ that uses $\xi$ before round $\sigma$ and $\rho$ at round $\sigma$ and thereafter. This rule is causal, because membership in the bad set is known before the current percept is generated. The models $\xi$ and $\widetilde\xi$ agree until the switch, so
\[
\TV(P_\xi^\pi,P_{\widetilde\xi}^\pi)\le\delta_\pi(c).
\]
For instance, couple their histories identically until the switch. They can differ only on episodes in which it occurs, whose probability is $\delta_\pi(c)$ under either model.

In the KL chain rule between $\widetilde\xi$ and $\rho$, all terms from the switch onward vanish. Before the switch, the history-action law agrees with that under $\xi$, restricted to paths that have not switched. With $k_t$ as in Proposition~\ref{prop:transfer}, this gives
\begin{align*}
\KL(P_{\widetilde\xi}^\pi\Vert P_\rho^\pi)
&=\sum_{t,h,a}P_\xi^\pi(\mathsf H_{<t}=h,\mathsf A_t=a,\sigma>t)k_t(h,a)\\
&\le\sum_{(t,h,a)\notin\mathcal B_{\pi,c}}d_{\xi,t}^\pi(h,a)k_t(h,a)
\le c\eta.
\end{align*}
Zero-occupancy terms are omitted. Before the switch, positive occupancy implies positive training occupancy, so finite training KL justifies the chain rule with finite local terms. Pinsker's inequality and the triangle inequality for total variation prove \eqref{eq:tailvalue}. Taking the infimum and applying the same two-policy value decomposition as in Theorem~\ref{thm:planning} proves \eqref{eq:tailplanning}.
\end{proof}

When $c\ge\kappa_\pi<\infty$, the bad set has zero deployment probability, and the theorem recovers the bound of Theorem~\ref{thm:planning}. Otherwise, it trades a smaller coverage multiplier against the probability of encountering a poorly covered branch. The theorem charges that probability once per episode, rather than the occupancy accumulated across rounds, which is what makes it informative for bounded returns. Proposition~\ref{prop:tailnecessary} below shows that neither of its two terms can be dropped.

\subsection{Counterexamples and failure modes}
\begin{example}[Exact on-policy prediction can coexist with suboptimality]\label{ex:coverage}
Take $H=1$ and a singleton environment class, so the true model equals $\xi$. The safe action gives reward $\Ber(1/2)$ under both $\xi$ and $\rho$. The explore action gives reward $\Ber(9/10)$ under $\xi$ and $\Ber(1/10)$ under $\rho$. If $\beta$ always chooses safe, then $\eta=0$ and the unique $\rho$-optimal action is safe, with true value loss $2/5$. The learned policy has $\kappa=1$, but the optimal comparator has an infinite coverage coefficient. On-policy accuracy alone is therefore insufficient, and some information or structural restriction on unobserved counterfactuals is needed for a general control guarantee.
\end{example}

\begin{example}[Square-root decision loss is unavoidable]\label{ex:sqrt}
Again take a singleton class and $H=1$. Both models predict $\Ber(1/2)$ for safe. Risky has law $\Ber(1/2-\Delta)$ under $\xi$ and $\Ber(1/2+\Delta)$ under $\rho$, where $0<\Delta<1/2$. Let behavior choose risky with probability $\epsilon\in(0,1/2]$. Its actual training KL is
\begin{align*}
\eta&=\epsilon\,\KL\bigl(\Ber(1/2-\Delta)\Vert\Ber(1/2+\Delta)\bigr)\\
&=\epsilon\,2\Delta\log_2\frac{1+2\Delta}{1-2\Delta}
=\epsilon\left(\frac{8}{\ln2}\Delta^2+O(\Delta^4)\right).
\end{align*}
The learned planner uniquely chooses risky and loses $\Delta$. The risky coefficient is $1/\epsilon$, and the safe coefficient $1/(1-\epsilon)$ is no larger. With $\kappa=1/\epsilon$, the loss is therefore $\Theta(\sqrt{\kappa\eta})$ as $\Delta\downarrow0$, and a gap-free $O(\kappa\eta)$ bound is impossible even with both policies covered.
\end{example}

\begin{example}[A simple action sequence can be exponentially rare]\label{ex:sequence}
Take $H=L$ binary actions, deterministic uninformative percepts before the terminal step, and a singleton environment class. All strings except $0^L$ have terminal reward $\Ber(1/2)$ under both models. After $0^L$, the true terminal reward is $\Ber(1/4)$ while $\rho$ predicts $\Ber(3/4)$. Independent uniform behavior reaches this string with probability $2^{-L}$, so
\[
\eta=2^{-L}\,\frac12\log_2 3.
\]
The learned planner chooses $0^L$ and loses $1/4$, and its maximum prefix occupancy ratio is exactly $2^L$. The string has a description of length at most $K(L)+O(1)$ under a fixed coding convention, despite its exponential behavioral rarity. Description complexity and behavioral coverage are thus different quantities, and a planner is drawn toward exactly the branch on which its model is optimistic.
\end{example}

\begin{example}[Bayesian value stability does not imply task-wise policy stability]\label{ex:ties}
Let $H=1$. In $\mu_0$, action 0 pays 1 and action 1 pays 0, and in $\mu_1$ the rewards are reversed. Give the two deterministic environments equal prior mass, so both mixture action values are $1/2$. Fix the exact planner's tie-breaking rule to choose action 0. Let $\rho$ agree with $\xi$ on action 0 and predict $\Ber(1/2+\delta)$ on action 1, where $0<\delta<1/2$. Under uniform behavior,
\[
\eta=-\frac14\log_2(1-4\delta^2)\longrightarrow0.
\]
Both deterministic policies have coverage coefficient 2. The learned planner chooses action 1, losing 1 in $\mu_0$ and gaining 1 in $\mu_1$ relative to the specified exact planner, and its Bayesian value loss is zero. Hence no task-wise policy-loss bound tending to zero with $\eta$ follows from the Bayesian bound alone, even with fixed positive task weights and finite coverage.
\end{example}

\begin{proposition}[Both terms of the coverage-failure bound are necessary]\label{prop:tailnecessary}
\begin{enumerate}[label=(\roman*),leftmargin=*,itemsep=1pt]
\item For every $q\in(0,1)$ there is a one-round instance with $\eta=0$, $\kappa_\pi=\infty$, $\delta_\pi(c)=q$ for every $c\ge1$, and $|V_\xi^\pi-V_\rho^\pi|=q$.
\item In Example~\ref{ex:sqrt}, $\delta_\pi(c)=0$ for both policies at $c=1/\epsilon$, while the learned planner loses $\Theta(\sqrt{\kappa\eta})$.
\end{enumerate}
\end{proposition}
\begin{proof}
For (i), take one environment and let $\beta$ always choose a safe action. Let $\pi$ choose an uncovered action with probability $q$. Both models pay zero after the safe action, while the uncovered action pays one under $\xi$ and zero under $\rho$. Then $\eta=0$, $\kappa_\pi=\infty$, the bad set contains exactly the uncovered pair for every $c\ge1$, and $|V_\xi^\pi-V_\rho^\pi|=q$. Part (ii) restates Example~\ref{ex:sqrt}, whose coefficients are at most $1/\epsilon$.
\end{proof}
An arbitrarily small probability of entering an uncovered branch therefore permits an equally small, but nonzero, value error, and the square-root dependence on $c\eta$ cannot be improved without a decision gap.

\subsection{A positive policy-value gap}
A margin must be defined on a policy class in which a positive gap is possible. Over all randomized policies, mixing an optimal and a suboptimal policy gives arbitrarily small positive value gaps.

\begin{proposition}[Gap-dependent planning bound]\label{prop:gap}
Let $\Pi$ be a finite nonempty policy class and let $\Pi^\star=\arg\max_{\pi\in\Pi}V_\xi^\pi$. Choose $\pi^\star\in\Pi^\star$ and assume $\Pi\setminus\Pi^\star$ is nonempty. Define
\[
\Delta_\Pi=\min_{\pi\in\Pi\setminus\Pi^\star}(V_\xi^{\pi^\star}-V_\xi^\pi)>0.
\]
Assume $\sup_{\pi\in\Pi}\kappa_\pi\le\kappa<\infty$ and let $\widehat\pi\in\arg\max_{\pi\in\Pi}V_\rho^\pi$. Under \eqref{eq:trainkl},
\begin{equation}
V_\xi^{\pi^\star}-V_\xi^{\widehat\pi}
\le\min\left\{1,\sqrt{2\ln2\,\kappa\eta},
\frac{2\ln2\,\kappa\eta}{\Delta_\Pi}\right\}.\label{eq:gap}
\end{equation}
If $\eta<\Delta_\Pi^2/(2\ln2\,\kappa)$, then $\widehat\pi\in\Pi^\star$.
\end{proposition}
\begin{proof}
Write $g=V_\xi^{\pi^\star}-V_\xi^{\widehat\pi}$ and $b=\sqrt{2\ln2\,\kappa\eta}$. Theorem~\ref{thm:planning} gives $g\le b$. If $g>0$, then $g\ge\Delta_\Pi$, so $g\le g^2/\Delta_\Pi\le b^2/\Delta_\Pi$. If $b<\Delta_\Pi$, positive regret is impossible.
\end{proof}
The conclusion preserves the optimal value rather than a particular policy among ties, and it is an algebraic corollary of the value bound rather than a finite-sample fast-rate theorem. More generally, any certified bound $g\le B$ gives $g\le\min\{B,B^2/\Delta_\Pi\}$ and preserves the optimal value when $B<\Delta_\Pi$, so one may use $B=\zeta+b_{\pi^\star}+b_{\widehat\pi}$ from Theorem~\ref{thm:rarecoverage} when its planning conditions hold.

\section{Application to approximate mentor deferral}\label{sec:deferral}
Golden Handcuffs motivates a concrete approximation question, namely whether a learned predictor preserves the decision to yield to a mentor. In its original formulation, the subjective reward range includes $-L$, and a sufficiently low optimal Bayesian value triggers an override \citep[Section~4]{ebtekarcohen2026}. The following corollary isolates threshold stability in our normalized finite-horizon setting. It is a consequence of the value bound rather than a new safety algorithm.

\begin{corollary}[A margin for preserving a deferral decision]\label{cor:deferral}
Assume \eqref{eq:trainkl}. Fix a nonempty policy class $\Pi$ and a shared return $G(\tau)\in[-L,1]$ with $L\ge0$. Define
\[
F_M=\sup_{\pi\in\Pi}\E_{P_M^\pi}G,
\qquad b=\sup_{\pi\in\Pi}b_\pi,
\]
where $b_\pi$ is the trajectory-TV certificate in Theorem~\ref{thm:rarecoverage}. Let a learned planning score satisfy $F_\rho-\zeta\le\widehat F_\rho\le F_\rho$, with $\zeta\ge0$ measured in units of $G$. Then
\begin{equation}
F_\xi-(L+1)b-\zeta\le\widehat F_\rho\le F_\xi+(L+1)b.\label{eq:deferralinterval}
\end{equation}
For any threshold $\theta\in\mathbb R$, the rules ``defer if $F_\xi\le\theta$'' and ``defer if $\widehat F_\rho\le\theta$'' agree whenever
\begin{equation}
|F_\xi-\theta|>(L+1)b+\zeta.\label{eq:deferralmargin}
\end{equation}
The conservative rule ``defer if $\widehat F_\rho\le\theta+(L+1)b$'' always defers whenever $F_\xi\le\theta$.
\end{corollary}
\begin{proof}
Apply Theorem~\ref{thm:rarecoverage} to $(G+L)/(L+1)\in[0,1]$. For every $\pi\in\Pi$, the value difference for $G$ is at most $(L+1)b_\pi\le(L+1)b$. Taking suprema gives $|F_\xi-F_\rho|\le(L+1)b$ without requiring attainment. Combining this with the interval for $\widehat F_\rho$ proves \eqref{eq:deferralinterval}. A perturbation of magnitude at most $(L+1)b+\zeta$ cannot cross a threshold farther away than that amount. Finally, if $F_\xi\le\theta$, the upper bound in \eqref{eq:deferralinterval} gives $\widehat F_\rho\le\theta+(L+1)b$.
\end{proof}

The factor $L+1$ makes a relevant sensitivity explicit. At a fixed probability-law certificate, a larger subjective loss range permits a larger value error. This does not show that increasing $L$ worsens an implemented agent's safety, since it also changes the values, the margins, and possibly the chosen policies. The uniform certificate over $\Pi$ can be demanding, because it must control every policy that matters to the optimized score, and it may be vacuous if the class contains poorly covered candidates. Section~\ref{sec:joint} shows where such certificates are hardest to obtain. Safety-critical histories may be rare under training. Theorem~\ref{thm:universal} shows that task rarity and policy rarity can then compound in the prediction error.

The corollary concerns one decision under the stated trajectory laws. To apply it after observations, we must update the mixture posterior and localize the training guarantee. The next proposition gives the required conditional budget. It is a KL chain-rule consequence, not an assumption that every history inherits the same error bound.

\subsection{Conditional continuation guarantees}\label{sec:continuation}
For $h=h_{<t}$ with $1\le t\le H$, write $P_{M,h}^{\pi}$ for the continuation law from rounds $t$ through $H$ under the kernels of $M$ and $\pi$, starting at $h$. This law is defined even if $\pi$ would not have reached $h$. Define
\[
p_\beta(h)=P_\xi^\beta(\mathsf H_{<t}=h),\qquad
\eta_h=\KL(P_{\xi,h}^{\beta}\Vert P_{\rho,h}^{\beta})
\]
when $p_\beta(h)>0$. The continuation of $\xi$ uses the updated weights $w_\mu(h)$ from \eqref{eq:posterior}.

\begin{proposition}[Localization of the training guarantee]\label{prop:continuation}
Under \eqref{eq:trainkl}, every history with $p_\beta(h)>0$ satisfies
\begin{equation}
\eta_h\le\frac{\eta}{p_\beta(h)}.\label{eq:conditionalbudget}
\end{equation}
More generally, for any prefix-free collection $\mathcal S$ of histories $h_{<t}$ with $1\le t\le H$,
\begin{equation}
\sum_{h\in\mathcal S:\,p_\beta(h)>0}p_\beta(h)\eta_h\le\eta.\label{eq:stoppedbudget}
\end{equation}
Here prefix-free means that no member is a proper prefix of another.

For a continuation policy $\pi$, define $d_{M,s}^{\pi\mid h}(g,a)$ from $P_{M,h}^{\pi}$ for $s=t,\ldots,H$. Let $\delta_{\pi,h}(c)$ be the $P_{\xi,h}^{\pi}$ probability of encountering a pair satisfying
\[
d_{\xi,s}^{\pi\mid h}(g,a)>c\,d_{\xi,s}^{\beta\mid h}(g,a).
\]
Then, for every $c\ge1$,
\begin{equation}
\TV(P_{\xi,h}^{\pi},P_{\rho,h}^{\pi})
\le\min\left\{1,\ \delta_{\pi,h}(c)+
\sqrt{\frac{\ln2}{2}c\eta_h}\right\}
\le\min\left\{1,\ \delta_{\pi,h}(c)+
\sqrt{\frac{\ln2}{2}\frac{c\eta}{p_\beta(h)}}\right\}.\label{eq:conditionaltv}
\end{equation}
\end{proposition}
\begin{proof}
Let $k_s(g,a)=\KL(\xi(\cdot\mid g,a)\Vert\rho(\cdot\mid g,a))$. The conditional chain rule gives
\[
p_\beta(h)\eta_h
=\sum_{s=t}^H\ \sum_{\substack{g\text{ extends }h\\a\in\mathcal A}}
d_{\xi,s}^{\beta}(g,a)k_s(g,a).
\]
The sum includes the current pair $g=h$ at $s=t$, and zero-occupancy terms are omitted. All positive-occupancy terms are finite because training KL is finite. For a prefix-free collection, the descendant pairs in these sums are disjoint. Their sum is therefore at most the full training KL, proving \eqref{eq:stoppedbudget} and its singleton case \eqref{eq:conditionalbudget}. Apply Theorem~\ref{thm:rarecoverage} on the continuation tree with training budget $\eta_h$, then substitute \eqref{eq:conditionalbudget}.
\end{proof}

The histories at which a deterministic history-dependent rule first requests a continuation certificate form such a prefix-free collection. Equation~\eqref{eq:stoppedbudget} therefore permits this form of adaptive selection. It does not permit adding the same budget bound over repeated requests at nested histories.

The factor $1/p_\beta(h)$ cannot be removed. Take two rounds and one action. Both models emit $\Ber(p)$ in round one, with $p\in(0,1)$. After percept $1$, let $\xi$ emit $0$ surely and let $\rho$ assign probability $2^{-d}$ to $0$, for any $d>0$. After percept $0$, both emit $0$ surely. With a singleton true environment, the episode KL is exactly $pd$, while continuation KL at history $1$ is $d$. For $d=1$, the conditional value error for terminal return $\mathbf1[e_2=0]$ is $1/2$, even as the episode KL tends to zero with $p$.

To apply Corollary~\ref{cor:deferral} at $h$, define
\[
b_{\pi,h}=\min\left\{1,\inf_{c\ge1}
\left[\delta_{\pi,h}(c)+\sqrt{\frac{\ln2}{2}c\eta_h}\right]\right\},
\qquad b_h=\sup_{\pi\in\Pi_h}b_{\pi,h}.
\]
For a nonempty continuation-policy class $\Pi_h$ and a shared return $G_h\in[-L,1]$, set $F_M(h)=\sup_{\pi\in\Pi_h}\E_{P_{M,h}^\pi}G_h$. If $F_\rho(h)-\zeta_h\le\widehat F_\rho(h)\le F_\rho(h)$ with $\zeta_h\ge0$, the same proof gives
\[
F_\xi(h)-(L+1)b_h-\zeta_h\le\widehat F_\rho(h)
\le F_\xi(h)+(L+1)b_h.
\]
Thus $|F_\xi(h)-\theta(h)|>(L+1)b_h+\zeta_h$ preserves the deferral decision at $h$. When only the episode budget is available, \eqref{eq:conditionaltv} gives the weaker certificate with $\eta/p_\beta(h)$. Histories with zero training probability have no such guarantee.

\subsection{Agreement of overrides does not imply event avoidance}
Even conditional certificates control scores rather than the actions taken when neither rule defers.

\begin{example}[Identical deferral decisions with different task-wise safety]\label{ex:deferralsafety}
Use Example~\ref{ex:ties}, take $G$ to be the one-round reward, and set the deferral threshold to zero. Both optimal values are positive: $F_\xi=1/2$ and $F_\rho=1/2+\delta$. Neither planner defers. In $\mu_0$, the specified ideal planner takes action $0$ and obtains reward $1$, while the learned planner takes action $1$ and obtains reward $0$. If failure is the event that reward equals zero, its probability is zero under the ideal policy and one under the learned policy. The deferral decisions agree although training KL tends to zero and both policies have coverage coefficient two.
\end{example}

This is a counterexample to inferring task-wise safety from threshold agreement. It is not a counterexample to Golden Handcuffs under its own assumptions. An extension to that agent still requires its semimeasure and discounting conventions, repeated mentor selection, and a guarantee for actions between overrides.

There is a further distinction between controlling risk estimates and excluding failures. For any finite-horizon event $\mathcal C'$, Theorem~\ref{thm:rarecoverage} implies
\[
|P_\xi^\pi(\mathcal C')-P_\rho^\pi(\mathcal C')|\le b_\pi.
\]
This compares two models under the same policy, with the Bayesian mixture as reference law. It does not compare the failure probabilities of two different policies in the true task, and a nonzero additive bound is not an event-avoidance guarantee. A useful next step is to identify conditions that preserve both the override and the safety of the actions permitted between overrides.

\section{Implications for learned agents}\label{sec:implications}
The readings below hold inside the normalized finite-horizon model, under assumption~\eqref{eq:trainkl} and the coverage conditions of the cited results. Each one states what that result implies for evaluation or deferral. None of them explains the behavior of a trained network. The bounds are worst-case transfer statements, and Appendix~\ref{sec:allocation} notes that they do not imply a gain from moving data between sources. Proposition~\ref{prop:context} gives the excess log loss of a discarded context. Proposition~\ref{prop:gap} makes the tolerable error depend on the value gap, and Example~\ref{ex:sqrt} shows that a vanishing gap leaves a square-root loss.

\paragraph{Rare histories.}
Proposition~\ref{prop:continuation} gives every history $h$ with $p_\beta(h)>0$ a continuation budget $\eta_h\le\eta/p_\beta(h)$. On a prefix-free collection the same training KL is shared across histories. A two-round example attains $\eta=p_\beta(h)\eta_h$, so this factor is tight. If only the episode budget is known, the conditional total-variation bound scales as $1/\sqrt{p_\beta(h)}$ before it saturates at one. Within this model, reliability after an unusual intermediate state is controlled by the conditional budget at that history, rather than by the episode KL alone.

\paragraph{Deferral.}
Corollary~\ref{cor:deferral} makes the ideal deferral rule and the learned rule agree only when $|F_\xi-\theta|>(L+1)b+\zeta$. The required separation scales with the width $L+1$ of the return range $[-L,1]$. At an observed history, the certificate from Proposition~\ref{prop:continuation} has an episode-budget form that scales as $1/\sqrt{p_\beta(h)}$ before saturating at one. Example~\ref{ex:deferralsafety} keeps both rules from deferring as the training KL tends to zero, with failure probability zero under the ideal action and one under the learned action. Within this model, a value-based override built on $\rho$ needs an explicit margin, such as the conservative shift in the corollary. Agreement with an ideal override rule does not give event avoidance.

\paragraph{Logged trajectories.}
Proposition~\ref{prop:transfer} carries training KL from the behavior policy $\beta$ to a deployment policy only through the occupancy ratio $\kappa_\pi$. Theorems~\ref{thm:planning} and~\ref{thm:rarecoverage} bound the value loss of a policy that planning selects under $\rho$. In Example~\ref{ex:sequence} that policy selects the one action string on which $\rho$ is wrong, and the episode KL equals $2^{-L}(\log_2 3)/2$. Example~\ref{ex:coverage} has $\eta=0$ when $\beta$ stays on the safe action, and the learned planner then loses $2/5$. Within this model, loss under $\beta$ does not certify the branches a planner may select, including branches that are rare under $\beta$. The passive form of this gap is the rarity penalty of \citet{liteng2026}, in which a small average KL can hide a large error on a rare environment.

\paragraph{Joint rarity.}
Proposition~\ref{prop:evidence} factors $\Gamma(\mu,\pi;\beta)$ into a policy likelihood ratio and the task evidence $\lambda_\mu(h)$ available before the prediction. If the task is unidentified on every evaluated history, Corollary~\ref{cor:extremes} gives $\Gamma\le\kappa_\pi$, with no factor $1/w_\mu$. If the posterior equals one on a branch that attains $\kappa_\pi$, the same corollary gives the product $\kappa_\pi/w_\mu$. Example~\ref{ex:laterevelation} puts the revealing percept after the rare action and still has $\Gamma=1/(w\epsilon)$, because a later prediction meets both factors. Theorem~\ref{thm:universal} produces excess error of order $\kappa_\pi\eta\,2^{K(\mu)}$ for infinitely many computable environments inside one fixed universal mixture. Within this model, a larger task weight leaves the policy ratio unchanged on branches that $\beta$ rarely visits, and those branches pay the product once the task is identifiable.

\paragraph{Separate data budgets.}
Proposition~\ref{prop:allocation} bounds the coverage of a rollout policy $\pi_j$ by $\kappa_{\pi_j}\le 1/\gamma_j$, where $\gamma_j$ is the probability of following $\pi_j$ for a whole episode. The family bound~\eqref{eq:familyindividual} puts the task weight in the denominator as well, and yields $\Gamma(\mu,\pi_j;\beta)\le 1/(\lambda q_\mu\gamma_j)$. Raising $\lambda$ puts more episodes on that family of tasks. Only raising some $\gamma_j$ changes the action sequences inside those episodes. Within this model, further tasks do not close a coverage gap left by a single behavior policy. A rollout under another policy counts through its $\gamma_j$ whether or not its return is high.

\section{Scope, limitations, and open problems}\label{sec:scope}
\paragraph{Learned predictors.}
The bounds hold for any response-model architecture, and whether a particular network attains \eqref{eq:trainkl} is a separate learning question. The ideal mixture supplies a reference distribution. Applying the planning results also requires a suitable policy class, a bound on planning error $\zeta$, and the shared return of Section~\ref{sec:planning}. A learned evaluator that differs from that return adds an error outside these bounds. Coverage refers to full histories, so context compression does not justify replacing it by coverage on compressed states.

Finite-transducer testbeds can make the assumptions explicit. With a finite task class, rational transition probabilities, rational prior and policy weights, rational returns, and a finite horizon, occupancies and policy values are exactly evaluable. Backward induction computes the optimal value over unrestricted policies. Finite KL values can be approximated to arbitrary precision, and support failure is decidable in this rational setting. Such testbeds make it possible to measure how close trained predictors come to the worst cases identified here, although a finite testbed cannot by itself certify the universal setting.

\paragraph{Normalized models and AIXI.}
The proofs use probability vectors, normalized trajectory laws, and a fixed bounded return, matching the sequence-measure mixtures of \citet[Section~2]{liteng2026}. Raw Solomonoff semimeasures can lose mass, so the proofs do not automatically apply to them. An AIXI-oriented extension must specify its causal semimeasure class, its treatment of termination or missing mass, its reward convention, and its horizon or discounting. Adding an absorbing termination percept is one possible modeling choice, but it changes the formal environment and requires checking the mixture and domination assumptions again. The analysis of \citet{wyethhutter2025ignorance} shows that this choice has decision-theoretic content, because assigning utility to terminated histories and treating missing mass as ignorance lead to different constructions. Section~\ref{sec:universal} sidesteps the issue by normalizing over estimable environments, which are almost surely productive.

\paragraph{Prediction and the control comparator.}
The planning theorems compare a learned-model policy with the optimal Bayesian policy in the same class. They do not compare either with a clairvoyant policy that knows $\mu$. Predictive convergence on visited histories does not establish the value of unvisited actions, and the choice of universal prior can substantially affect behavior \citep{leikehutter2015}. \citet{leike2016} prove asymptotic optimality in mean for an exploratory Thompson-sampling agent in general environments, and sublinear regret under recoverability and suitable discount-function assumptions. Combining them with the present fixed-policy bounds to obtain an asymptotic control theorem is a natural direction.

\paragraph{Horizon and data.}
Normalization removes an explicit reward-sum factor from the planning bound, but $\eta$, $\kappa_\pi$, and $\Gamma$ can still grow with $H$, and Example~\ref{ex:sequence} gives exponential coverage growth. The clipped and coverage-failure bounds avoid requiring a finite maximum ratio, although their residual quantities can also grow with $H$. Finite-sample estimates of \eqref{eq:trainkl} and of coverage require separate assumptions, especially for adaptive data. Appendix~\ref{app:computability} shows why one fixed uniformly computable predictor cannot have bounded population trajectory KL at all horizons against a normalized mixture containing every computable deterministic sequence. This obstruction complements the computability classifications of \citet{leikehutter2018} and does not contradict the restricted-source guarantees for speed priors in \citet{filan2016}.

\paragraph{Open problems.}
Four questions remain. First, Corollary~\ref{cor:computablelower} supplies computable witnesses without resource bounds. Which restrictions on runtime, representation size, or the learning procedure alter the worst-case penalty? Second, $\Gamma$ is optimal for arbitrary local losses but may overstate the cost for a particular return. Loss-sensitive coverage quantities could improve that comparison.

Third, Proposition~\ref{prop:continuation} controls conditional scores, but repeated overrides still require control of the actions taken between them. The safety counterexample shows why score agreement alone is insufficient. Fourth, an AIXI extension must specify the semimeasure and value conventions discussed above. These questions concern additional structure beyond the current population KL assumption.

\section{Conclusion}
When a learned predictor supplies an agent's world model, the agent's actions change where prediction error is evaluated. The joint occupancy coefficient separates the policy likelihood ratio from task evidence available before each prediction. Indistinguishability throughout the relevant histories removes the prior-mass multiplier. In contrast, one fixed universal causal mixture admits a matching multiplicative lower bound on infinitely many computable environments. Finite-horizon rational approximation preserves this order for computable predictors, so computability alone does not remove the penalty.

For bounded-return planning, a coverage-failure probability can replace an infinite maximum ratio. Conditional continuation bounds quantify the extra cost of assessing a rare observed history. These results give margins for preserving mentor-deferral decisions, while an explicit counterexample separates those decisions from task-wise event avoidance. Section~\ref{sec:implications} collects the consequences for evaluation and deferral that follow inside the model. Because the guarantees are stated for population laws, they apply to any agent that meets their assumptions. Verifying those assumptions for trained agents and extending the analysis to the semimeasure-based decisions of AIXI are natural next steps.

\bibliographystyle{plainnat}
\bibliography{references}

\appendix
\section{Supporting results}\label{app:additional}
This appendix collects results that refine or complement the main analysis. They use the setup of Section~\ref{sec:setup}.

\subsection{The prediction cost of discarded context}\label{sec:context}
A finite context window or deterministic memory representation can prevent a response model from retaining information needed for Bayes-optimal prediction. The following conditional-information identity makes this source of approximation error explicit. It concerns the information retained by the representation, independently of how a network implements the resulting predictor.

Let $\mathsf H_{<t},\mathsf A_t,\mathsf E_t$ be the random history, action, and percept under $P_\xi^\beta$. Fix a deterministic representation $U_t=\phi_t(\mathsf H_{<t})$, with time available to the predictor, and write
\[
\bar\xi_t^\beta(e\mid u,a)
=P_\xi^\beta(\mathsf E_t=e\mid U_t=u,\mathsf A_t=a).
\]
Choose arbitrary normalized conditionals at pairs with zero probability, and define
\begin{equation}
\mathcal I_\beta(\phi)
=\sum_{t=1}^H I_2(\mathsf E_t;\mathsf H_{<t}\mid U_t,\mathsf A_t),\label{eq:contextinformation}
\end{equation}
where $I_2$ is conditional mutual information in bits, evaluated under $P_\xi^\beta$.

\begin{proposition}[Context representation and excess log loss]\label{prop:context}
For any normalized predictor of the form $\rho(e\mid h,a)=r_t(e\mid\phi_t(h),a)$,
\begin{align}
\KL(P_\xi^\beta\Vert P_\rho^\beta)
={}&\mathcal I_\beta(\phi)\notag\\
&+\sum_{t=1}^H\E_{P_\xi^\beta}
\KL\!\left(\bar\xi_t^\beta(\cdot\mid U_t,\mathsf A_t)
\Vert r_t(\cdot\mid U_t,\mathsf A_t)\right).\label{eq:contextdecomposition}
\end{align}
Consequently, the smallest trajectory KL over all such predictors equals $\mathcal I_\beta(\phi)$, and Assumption~\eqref{eq:trainkl} requires $\eta\ge\mathcal I_\beta(\phi)$ for this representation class.
\end{proposition}
\begin{proof}
In the log-likelihood ratio at time $t$, insert the compressed conditional,
\[
\log_2\frac{\xi(\mathsf E_t\mid\mathsf H_{<t},\mathsf A_t)}
{r_t(\mathsf E_t\mid U_t,\mathsf A_t)}
=\log_2\frac{\xi(\mathsf E_t\mid\mathsf H_{<t},\mathsf A_t)}
{\bar\xi_t^\beta(\mathsf E_t\mid U_t,\mathsf A_t)}
+\log_2\frac{\bar\xi_t^\beta(\mathsf E_t\mid U_t,\mathsf A_t)}
{r_t(\mathsf E_t\mid U_t,\mathsf A_t)}.
\]
The first expected term is conditional mutual information because $U_t$ is a function of the history. Conditioning the second on $(U_t,\mathsf A_t)$ gives the corresponding KL term in \eqref{eq:contextdecomposition}. Summing and applying the trajectory chain rule proves the identity, and taking $r_t=\bar\xi_t^\beta$ attains the minimum. If $r_t$ assigns zero mass to an event of positive training probability, both sides are infinite and the identity holds in the extended sense.
\end{proof}

For example, consider a singleton environment, one action, and two rounds with $\mathsf E_1$ a fair bit and $\mathsf E_2=\mathsf E_1$. If $\phi_2$ discards the first percept, then $\mathcal I_\beta(\phi)=1$ bit, while retaining it reduces the minimum KL to zero. Additional training cannot remove the one-bit penalty while that representation is fixed. More generally, zero penalty means that the representation, together with the chosen action, is sufficient for predicting responses on the training distribution. It does not establish sufficiency on histories that the behavior policy never visits, and the compressed conditional itself depends on $\beta$ because compression pools histories according to their training frequencies.

\subsection{Average-case prediction and clipped transfer}
\begin{proposition}[Causal average-case decomposition]\label{prop:variance}
For every common policy $\pi$,
\begin{equation}
\E_{\mu\sim w}S_H^\pi(\mu,\rho)
=\E_{\mu\sim w}S_H^\pi(\mu,\xi)+S_H^\pi(\xi,\rho).\label{eq:variance}
\end{equation}
Under \eqref{eq:trainkl} and finite $\kappa_\pi$, the average excess is at most $(\ln2)\kappa_\pi\eta$.
\end{proposition}
\begin{proof}
At a pair $(h,a)$ with positive mixture occupancy,
\[
\frac{w_\mu d_{\mu,t}^\pi(h,a)}{d_{\xi,t}^\pi(h,a)}=w_\mu(h).
\]
Thus \eqref{eq:posterior} is the conditional mean of the task's probability vector. Expanding the squared norm gives
\begin{align*}
\sum_\mu w_\mu(h)\norm{\mu(\cdot\mid h,a)-\rho(\cdot\mid h,a)}^2
={}&\sum_\mu w_\mu(h)\norm{\mu(\cdot\mid h,a)-\xi(\cdot\mid h,a)}^2\\
&+\norm{\xi(\cdot\mid h,a)-\rho(\cdot\mid h,a)}^2.
\end{align*}
The cross term vanishes because $\sum_\mu w_\mu(h)(\mu-\xi)=0$. Multiply by $d_{\xi,t}^\pi(h,a)$ and sum. All squared errors are bounded, so countable task averaging causes no integrability problem. Lemma~\ref{lem:l2}, the chain rule, and Proposition~\ref{prop:transfer} give the excess bound.
\end{proof}
The identity compares two predictors under the same policy. It does not compare the losses on trajectories generated by two different model-optimal policies. Together with Theorem~\ref{thm:universal}, it shows that the gap between average and task-wise robustness identified by \citet{liteng2026} persists under interaction, where both the average and the task-wise excess scale with the policy coefficient.

\begin{proposition}[Prediction transfer with clipped occupancy]\label{prop:clipped}
For $c\ge0$, define the occupancy mass exceeding $c$ times the training mass by
\[
\Lambda_{\mu,\pi}(c)
=\sum_{t,h,a}\bigl[d_{\mu,t}^\pi(h,a)-c\,d_{\xi,t}^\beta(h,a)\bigr]_+,
\qquad [x]_+=\max\{x,0\}.
\]
Under \eqref{eq:trainkl}, without assuming finite $\Gamma$,
\begin{equation}
A_H^{\mu,\pi}(\xi,\rho)
\le\inf_{c\ge0}\left\{(\ln2)c\eta+2\Lambda_{\mu,\pi}(c)\right\}
\le2H.\label{eq:clippedprediction}
\end{equation}
This bound can replace $(\ln2)\Gamma\eta$ for the approximation component in the triangle-inequality proof of Theorem~\ref{thm:joint}.
\end{proposition}
\begin{proof}
Write $p=d_{\mu,t}^\pi(h,a)$, $q=d_{\xi,t}^\beta(h,a)$, and $\ell=\norm{\xi(\cdot\mid h,a)-\rho(\cdot\mid h,a)}^2\le2$. The inequality $p\ell\le cq\ell+2[p-cq]_+$, summed over pairs, gives
\[
A_H^{\mu,\pi}(\xi,\rho)
\le c\,S_H^\beta(\xi,\rho)+2\Lambda_{\mu,\pi}(c)
\le(\ln2)c\eta+2\Lambda_{\mu,\pi}(c).
\]
Take the infimum. At $c=0$, $\Lambda_{\mu,\pi}(0)=H$, which proves the cap. If $\Gamma<\infty$, taking $c=\Gamma$ recovers \eqref{eq:approxbound}.
\end{proof}
The residual counts occupancy across rounds, as appropriate for cumulative prediction loss. Theorem~\ref{thm:rarecoverage} instead charges the probability of entering a poorly covered branch once per episode.

\subsection{Simultaneous prior and policy shift}\label{sec:shift}
Let $w'$ be another prior on $\cM$, with mixture $\xi'$ defined by its own posterior weights. Suppose $w'_\mu\le Cw_\mu$ for all $\mu$, where $C\ge1$. Define
\[
\Gamma'=\max_{t,h,a}\frac{d_{\xi',t}^\pi(h,a)}{d_{\xi,t}^\beta(h,a)}.
\]
Marginal domination yields $\Gamma'\le C\kappa_\pi$.

\begin{theorem}[Simultaneous prior and policy shift]\label{thm:shift}
Under \eqref{eq:trainkl}, if $\Gamma'<\infty$, then
\begin{align}
&\E_{w'}S_H^\pi(\mu,\rho)-\E_{w'}S_H^\pi(\mu,\xi')
 =S_H^\pi(\xi',\rho)\notag\\
&\qquad\le(\ln2)\left(\sqrt{\KL(w'\Vert w)}+\sqrt{\Gamma'\eta}\right)^2.\label{eq:shiftsharp}
\end{align}
If also $\kappa_\pi<\infty$, the right-hand side is at most
\begin{equation}
(\ln2)\left(\sqrt{\log_2 C}+\sqrt{C\kappa_\pi\eta}\right)^2.\label{eq:shiftcoarse}
\end{equation}
\end{theorem}
\begin{proof}
Apply Proposition~\ref{prop:variance} to the prior $w'$. On the joint space of task labels and trajectories, the KL between $w'_\mu P_\mu^\pi(\tau)$ and $w_\mu P_\mu^\pi(\tau)$ equals $\KL(w'\Vert w)$. Marginalizing the label and applying the log-sum inequality give
\[
\KL(P_{\xi'}^\pi\Vert P_\xi^\pi)\le\KL(w'\Vert w)\le\log_2 C.
\]
The first inequality also follows directly from the KL chain rule for the latent task. Hence $S_H^\pi(\xi',\xi)\le(\ln2)\KL(w'\Vert w)$. Meanwhile,
\[
\sum_{t,h,a}d_{\xi',t}^\pi(h,a)\norm{\xi(\cdot\mid h,a)-\rho(\cdot\mid h,a)}^2
\le(\ln2)\Gamma'\eta.
\]
The weighted Euclidean triangle inequality, now with occupancy under $\xi'$, proves \eqref{eq:shiftsharp}. Substituting the two domination bounds gives \eqref{eq:shiftcoarse}.
\end{proof}
The baseline is the test-prior mixture $\xi'$. Even when $\rho=\xi$, changing the prior can leave excess prediction error, and the term $\KL(w'\Vert w)$ records that mismatch. With one available action, \eqref{eq:shiftcoarse} recovers the finite-horizon form of \citet[Theorem~5]{liteng2026}.

\subsection{Episode-level allocation of training data}\label{sec:allocation}
Suppose behavior first samples $J$ with probabilities $\gamma_j$, independently of the task, and follows $\pi_J$ for the whole episode. This can be represented by one history-dependent behavior policy common to all environments. For a history $h=h_{<t}$, let
\[
Q_j(h)=\prod_{s<t}\pi_j(a_s\mid h_{<s}).
\]
Whenever the denominator is positive, its behavioral representation is
\begin{equation}
\beta(a\mid h)=\frac{\sum_j\gamma_j Q_j(h)\pi_j(a\mid h)}{\sum_j\gamma_j Q_j(h)}.\label{eq:episodicbeta}
\end{equation}
At other histories, choose any normalized action distribution. This formula follows by cancelling the common environment likelihood when conditioning on the observed history, and it generally differs from resampling $J$ independently at each step. The behavior policy $\beta_\epsilon$ of Theorem~\ref{thm:universal} is an instance.

\begin{proposition}[Episode-level policy allocation]\label{prop:allocation}
For the policy in \eqref{eq:episodicbeta} and any causal model $M$,
\[
P_M^\beta=\sum_j\gamma_jP_M^{\pi_j}.
\]
Consequently, for every $j$ with $\gamma_j>0$,
\[
\kappa_{\pi_j}\le\gamma_j^{-1},\qquad
\Gamma(\mu,\pi_j;\beta)\le(w_\mu\gamma_j)^{-1}.
\]
\end{proposition}
\begin{proof}
Multiplying \eqref{eq:episodicbeta} along a trajectory telescopes the denominator and numerator, leaving $\sum_j\gamma_j\prod_t\pi_j(a_t\mid h_{<t})$. Multiplying by the common environment likelihood gives the trajectory identity. Marginalize to each history-action pair and use mixture domination over tasks.
\end{proof}

To make task-family allocation precise, let $q$ and $v$ be priors and set $w=\lambda q+(1-\lambda)v$ for $\lambda\in(0,1]$, discarding any zero-weight tasks so that the support convention of Section~\ref{sec:setup} holds. For an individual task with $q_\mu>0$,
\begin{equation}
\Gamma(\mu,\pi_j;\beta)\le\frac1{\lambda q_\mu\gamma_j}.\label{eq:familyindividual}
\end{equation}
Family mass $\lambda$ alone therefore does not guarantee an individual-task coefficient $1/(\lambda\gamma_j)$. That coefficient instead controls the family mixture. With $\xi_q$ denoting the mixture under $q$,
\[
\max_{t,h,a}\frac{d_{\xi_q,t}^{\pi_j}(h,a)}{d_{\xi,t}^\beta(h,a)}\le\frac1{\lambda\gamma_j},
\qquad \KL(q\Vert w)\le\log_2\frac1\lambda,
\]
so Theorem~\ref{thm:shift} bounds the family-average excess relative to $\xi_q$ by
\[
(\ln2)\left(\sqrt{\log_2\frac1\lambda}+\sqrt{\frac{\eta}{\lambda\gamma_j}}\right)^2.
\]
These statements distinguish task allocation from interaction allocation. Increasing the frequency of a task family raises $\lambda$, while collecting episodes with a particular rollout policy raises its $\gamma_j$, and both enter the transfer guarantee. Adding more task descriptions alone need not expose the responses to alternative action sequences. Adaptive generation of training tasks, as in \citet{cowsik2026}, changes the first allocation, while collecting responses to additional action sequences changes the second. The coefficients concern coverage under the environment law used in the assumption, so model-generated responses require a separate argument that they represent that law. Changing the data distribution can also change the attainable approximation error, so the inequalities do not imply that reallocating data monotonically improves a trained predictor.

Pessimistic planning is a complementary direction. A lower confidence objective can compare against a covered policy without requiring comparable coverage of every candidate policy, provided the confidence set and model assumptions are valid \citep{uehara2022}. Theorem~\ref{thm:rarecoverage} instead bounds a plug-in planner using the coverage-failure probabilities of both policies, and it does not supply the calibrated confidence set that pessimism needs. A complexity-aware pessimistic algorithm for universal prediction requires additional work.

\subsection{An explicit finite-class instance}
\begin{proposition}[Finite-class lower bound for joint rarity]\label{prop:jointlower}
For every $w\in(0,1)$, $\epsilon\in(0,1]$, and $\delta\in(0,1/4]$, there is a two-stage, two-environment instance with a target task of mass $w$ such that
\[
\kappa_\pi=\epsilon^{-1},\qquad \Gamma=(w\epsilon)^{-1},\qquad
\eta=w\epsilon\left[-\frac12\log_2(1-4\delta^2)\right],
\]
where $\eta$ is the actual training KL, and
\begin{equation}
S_2^\pi(\mu,\rho)-S_2^\pi(\mu,\xi)
=A_2^{\mu,\pi}(\xi,\rho)=2\delta^2
\ge\frac{3\ln2}{4}\Gamma\eta.\label{eq:jointlower}
\end{equation}
\end{proposition}
\begin{proof}
Let $\mathcal A=\{s,r\}$ and $\mathcal E=\{0,1\}$. Both policies first choose $s$. The target environment $\mu_1$ outputs $e_1=1$ deterministically, and $\mu_0$ outputs $e_1=0$. Give $\mu_1$ prior mass $w$. At the second step, both environments output a fair bit under either action. The behavior policy chooses $r$ with probability $\epsilon$ at that step, independently of $e_1$, whereas $\pi$ always chooses $r$.

Let $\rho$ agree with $\xi$ except that after $e_1=1$ and action $r$ it predicts $\Ber(1/2+\delta)$. The exceptional pair has training occupancy $w\epsilon$ and target deployment occupancy one. All other reached pairs have no larger occupancy ratios, which gives the stated coefficients. The only nonzero local approximation KL is
\[
\KL(\Ber(1/2)\Vert\Ber(1/2+\delta))=-\frac12\log_2(1-4\delta^2).
\]
The first-step predictor is unchanged, and the target and mixture conditionals agree at the second step, so the task-wise excess is exactly $2\delta^2$. Finally, $-\ln(1-x)\le x/(1-x)$ for $0\le x<1$ and $4\delta^2\le1/4$, so
\[
\Gamma\eta\le\frac{8\delta^2}{3\ln2}.
\]
This proves \eqref{eq:jointlower}.
\end{proof}
The construction is the early-revelation case of Corollary~\ref{cor:extremes} with explicit constants at a fixed horizon of two. Theorem~\ref{thm:universal} establishes the same dependence, with constants depending on the reference machine, inside one fixed universal mixture.

\section{An infinite-horizon computability obstruction}\label{app:computability}
The following elementary diagonal argument gives a population-KL obstruction for one fixed computable predictor. It concerns a different quantity from the arithmetical-hierarchy classifications in \citet{leikehutter2018}.
\begin{proposition}[No uniformly bounded universal KL for a fixed computable predictor]\label{prop:diagonal}
In the passive binary special case, let $\rho$ have conditional probabilities uniformly computable from the finite history. Let $\xi$ be a normalized mixture assigning positive weight to every computable deterministic binary sequence. Then there is a constant $w>0$ such that for every $H$,
\[
\KL(P_\xi^H\Vert P_\rho^H)\ge wH\log_2(8/5)-1.
\]
In particular, these divergences cannot remain bounded as $H\to\infty$.
\end{proposition}
\begin{proof}
At a prefix $x_{<t}$, compute a rational approximation $q_t$ to $\rho(1\mid x_{<t})$ with error at most $1/8$. Set $x_t=1$ if $q_t\le1/2$ and $x_t=0$ otherwise. In both cases $\rho(x_t\mid x_{<t})\le5/8$. This constructs a computable deterministic sequence, whose environment has some mixture weight $w>0$.

For the event $B_H=\{x_{1:H}\}$, write $p=P_\xi^H(B_H)\ge w$ and $q=P_\rho^H(B_H)\le(5/8)^H$. If $q=0$, the claimed divergence is infinite. Otherwise, data processing to the indicator of $B_H$ yields
\begin{align*}
\KL(P_\xi^H\Vert P_\rho^H)
&\ge p\log_2\frac1q+(1-p)\log_2\frac1{1-q}-h_2(p)\\
&\ge wH\log_2(8/5)-1,
\end{align*}
where $h_2(p)=-p\log_2p-(1-p)\log_2(1-p)\le1$. The construction uses approximate probabilities, so it never requires deciding equality of computable real numbers.
\end{proof}
The obstruction already holds without interaction. It does not rule out horizon-dependent predictors, finite-horizon guarantees, or restricted computable environment classes.

\end{document}